\documentclass[runningheads]{llncs}

\usepackage{amsthm} 
\usepackage[T1]{fontenc}
\usepackage{graphicx}
\usepackage{booktabs}
\usepackage[misc]{ifsym}

\usepackage{mwe}
\usepackage{booktabs}
\usepackage{siunitx} 
\usepackage{xcolor} 
\usepackage{makecell} 
\usepackage{verbatim}
\usepackage{appendix}
\usepackage{bbding}
\usepackage{svg}

\usepackage{amsmath,amsfonts,bm}

\def\eqref#1{equation~\ref{#1}}

\def\1{\bm{1}}

\DeclareMathAlphabet{\mathsfit}{\encodingdefault}{\sfdefault}{m}{sl}
\SetMathAlphabet{\mathsfit}{bold}{\encodingdefault}{\sfdefault}{bx}{n}

\usepackage{hyperref}
\usepackage{url}
\usepackage{graphicx} 
\usepackage{epstopdf}
\usepackage{subcaption}
\usepackage{arydshln}
\usepackage{multirow}

\theoremstyle{definition}

\theoremstyle{remark}
\newcommand{\method}{SubGEC}

\newcommand{\ie}{\textit{i}.\textit{e}., }
\newcommand{\eg}{\textit{e}.\textit{g}., }
\DeclareMathOperator{\diag}{diag}

\usepackage[misc]{ifsym}

\begin{document}

\title{Subgraph Gaussian Embedding Contrast for Self-Supervised Graph Representation Learning}

\titlerunning{Subgraph Gaussian Embedding Contrast for Self-Supervised GRL}


\author{Shifeng Xie \orcidID{0009-0005-2909-2978}\and Aref Einizade \orcidID{0000-0002-8546-7261} \and \\  Jhony H. Giraldo (\Letter) \orcidID{0000-0002-0039-1270}}


\authorrunning{S. Xie et al.}

\institute{LTCI, Télécom Paris, Institut Polytechnique de Paris, Palaiseau, France \email{\{shifeng.xie, aref.einizade, jhony.giraldo\}@telecom-paris.fr}
}




\maketitle              

\begin{abstract}
Graph Representation Learning (GRL) is a fundamental task in machine learning, aiming to encode high-dimensional graph-structured data into low-dimensional vectors.
Self-Supervised Learning (SSL) methods are widely used in GRL because they can avoid expensive human annotation.
In this work, we propose a novel Subgraph Gaussian Embedding Contrast (\method) method.
Our approach introduces a subgraph Gaussian embedding module, which adaptively maps subgraphs to a structured Gaussian space, ensuring the preservation of input subgraph characteristics while generating subgraphs with a controlled distribution.
We then employ optimal transport distances, more precisely the Wasserstein and Gromov-Wasserstein distances, to effectively measure the similarity between subgraphs, enhancing the robustness of the contrastive learning process.
Extensive experiments across multiple benchmarks demonstrate that \method~outperforms or presents competitive performance against state-of-the-art approaches.
Our findings provide insights into the design of SSL methods for GRL, emphasizing the importance of the distribution of the generated contrastive pairs.

\end{abstract}

\section{Introduction}

Graph Representation Learning (GRL) is a fundamental task in machine learning and data mining, aiming to encode high-dimensional, sparse graph-structured data into low-dimensional dense vectors \cite{ju2024comprehensive}.
Effective GRL techniques enable downstream applications such as node classification, link prediction, and community detection.
Self-Supervised Learning (SSL) has emerged as a promising approach for GRL by reducing the dependence on extensive human annotation \cite{jaiswal2020survey}.
Among SSL methods, contrastive learning has gained significant attention due to its ability to learn meaningful representations by distinguishing similarities and differences among data samples.
In contrastive learning, positive sample pairs typically consist of two augmented views of the same data point, which should be mapped close in the representation space, whereas negative sample pairs are formed by comparing different data points \cite{chen2020simple}.

\begin{figure}[t]
    \centering
    \includegraphics[width=0.95\textwidth]{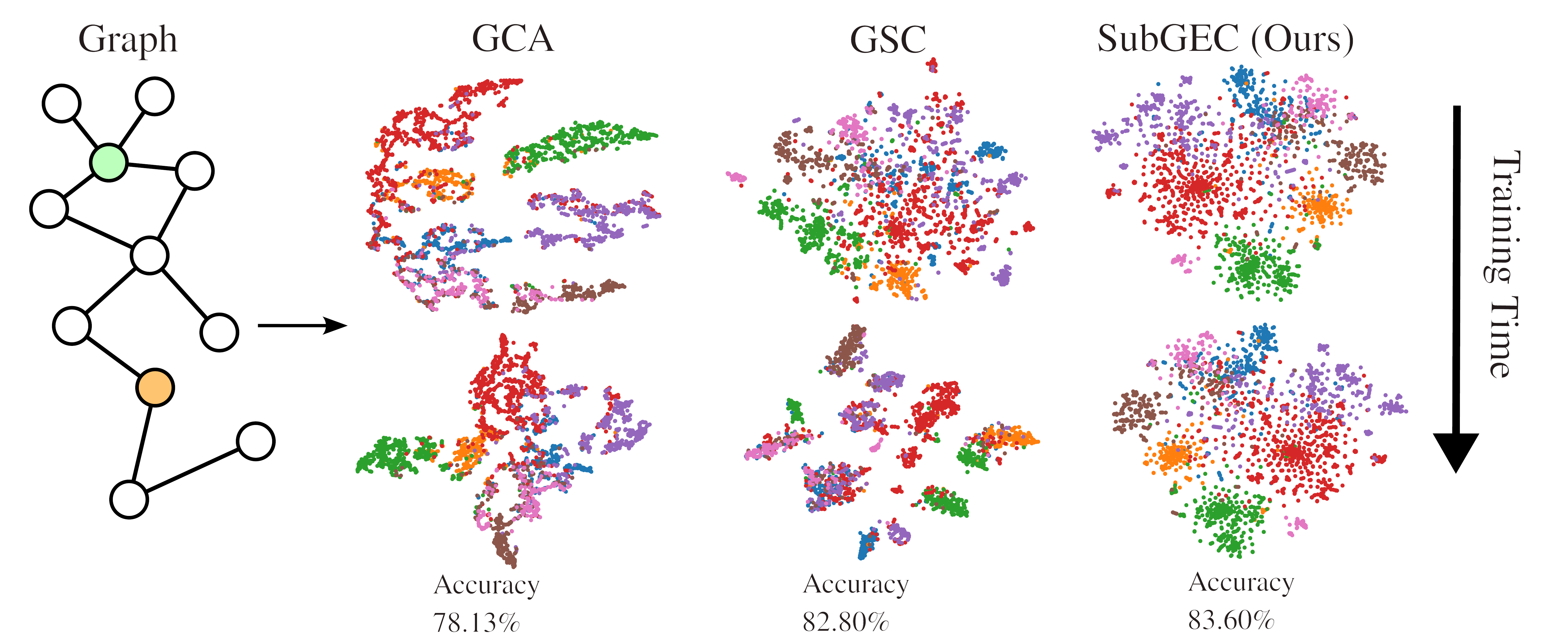} 
    \caption{t-stochastic neighbor embedding (t-SNE) visualizations of previous graph representation learning methods based on contrastive learning: GCA \cite{GCA}, GSC \cite{GSC}, and our method \method.
    Each point corresponds to a node representation with reduced dimensionality, with colors indicating classes.
    Unlike GCA and GSC, which exhibit sharp boundaries, \method~maps node representations into a dense, uniform, and linearly separable space.}
    \label{fig:graph_representation_comparison}
\end{figure}

Existing graph-based contrastive learning methods primarily generate positive and negative pairs through structural perturbations \cite{GRACE,thakoorlarge,GCA} or learnable transformations \cite{zhuounified,GSC}. 
However, Figure \ref{fig:graph_representation_comparison} shows t-stochastic neighbor embedding (t-SNE) visualizations of previous SSL methods, such as GCA \cite{GCA} and GSC \cite{GSC}, where we observe uneven node distributions, sharp boundaries, and erroneous clusters. 
These issues suggest that existing approaches struggle to maintain smooth and meaningful representations, negatively impacting their performance in GRL tasks.

In this paper, we propose the \textbf{Sub}graph \textbf{G}aussian \textbf{E}mbedding \textbf{C}ontrast (\method) model, a novel framework for graph contrastive learning.
Our method introduces the Subgraph Gaussian Embedding (SGE) module, which maps input subgraphs to a structured Gaussian space, ensuring that the output features follow a Gaussian distribution using Kullback--Leibler (KL) divergence. 
This learnable mapping effectively controls the distribution of embeddings, improving representation quality.
The generated subgraphs are then paired with the original subgraphs to form positive and negative contrastive pairs, and similarity is measured using Optimal Transport (OT) distances.
By leveraging the Wasserstein and Gromov-Wasserstein distances, our approach enhances robustness and mitigates mode collapse (also called positive collapse \cite{jing2022understandingdimensionalcollapsecontrastive}), where the embeddings shrink into a low-dimensional subspace, by controlling the embedding distribution.

Gaussian distributions provide several properties that make them useful in SSL for graphs.
For example, they preserve important mathematical structures, such as displacement interpolation, which helps in clustering and interpolation tasks \cite{zhu2020optimal,goldfeld2020gaussian}.
Gaussian smoothing also improves robustness by reducing noise and stabilizing learned representations \cite{goldfeld2020gaussian}.
Additionally, their simple parameterization using means and covariances makes them computationally efficient, enabling scalability to high-dimensional spaces while avoiding excessive computational costs \cite{zhu2020optimal,chen2018optimal}.
These properties make Gaussian-based OT a valuable tool in fields such as machine learning, physics, and statistical inference \cite{genevay2016stochastic}.
In this work, we theoretically and empirically prove the benefits of using Gaussian embeddings in contrastive learning.


The primary contributions of this paper are as follows:
\begin{itemize}
    \item We introduce \method, a novel framework that outperforms or remains competitive with state-of-the-art methods across eight benchmark datasets.
    \item We theoretically and empirically highlight the importance of mapping the distribution of contrastive pairs into a Gaussian space and analyze its impact on GRL. 
    \item We conduct extensive ablation and validation studies to demonstrate the effectiveness of each component of \method.
\end{itemize}

\section{Related Work}

GRL has gained significant attention due to its ability to encode structured data into meaningful representations.
Here, we review the recent advancements in Graph Neural Networks (GNNs), SSL on graphs, and contrastive learning techniques. 

\textbf{GNNs} \cite{wu2020comprehensive} have been widely adopted for learning representations that capture both node features and graph topology \cite{ju2024comprehensive}.
Several architectures have been proposed to improve their learning capabilities. 
For example, Graph Convolutional Networks (GCNs) \cite{kipf2016semi} leverage a simplification of graph filters to aggregate information from neighboring nodes.
GraphSAGE \cite{hamilton2017inductive} introduced an inductive learning framework with multiple aggregation functions, enabling generalization to unseen nodes.
Furthermore, Graph Attention Networks (GAT) \cite{velivckovic2017graph} integrate attention mechanisms to dynamically weigh node relationships, improving feature propagation.
However, these models require supervised training.


\textbf{SSL on graphs} aims to design and solve learning tasks that do not require labeled data, avoiding costly supervised learning methodologies. 
Based on how these tasks are defined, SSL methods can be categorized into two main types: \textit{predictive} and \textit{contrastive} approaches.

Predictive methods focus on learning useful representations by generating perturbed versions of the input graph. For example, BGRL \cite{thakoorlarge} learns node representations by encoding two perturbed versions of a graph using an online encoder and a target encoder.
The online encoder is optimized to predict the target encoder's representation, while the target encoder is updated as an exponential moving average of the online encoder. 
BNLL \cite{liu2024bootstrap} improves upon BGRL by introducing additional positive node pairs based on a homophily assumption, where neighboring nodes tend to share the same label.
This is achieved by incorporating cosine similarity between a node's online representation and the weighted target representations of its neighbors.
VGAE \cite{vgae} adopts a variational autoencoder framework to reconstruct the input graph and its features. 

\textbf{Contrastive methods}, which are the focus of this paper, generally outperform predictive methods in SSL for graphs.
These methods can be classified based on how data pairs are defined: node-to-node, graph-to-graph, and node-to-graph comparisons.
For example, GRACE \cite{GRACE} generates two perturbed graph views and applies contrastive learning at the node level.
MUSE \cite{MUSE} refines this approach by extracting multiple embeddings—semantic, contextual, and fused—to enhance node-to-node contrastive learning. 
However, node-level contrastive learning is often suboptimal as it struggles to capture the overall structural information of the graph.

At the subgraph level, DGI \cite{DGI} employs node-to-graph contrast, where it extracts node embeddings from the original and perturbed graphs and adjusts their agreement levels using a readout function. 
Spectral polynomial filter methods like GPR-GNN \cite{chien2021adaptive} and ChebNetII \cite{he2024convolutionalneuralnetworksgraphs} offer greater flexibility than GCNs by adapting to different homophily levels.
However, they often underperform when used as encoders for traditional SSL methods.
To address this, PolyGCL \cite{polygcl} constrains polynomial filter expressiveness to construct high-pass and low-pass graph views while using a simple linear combination strategy for optimization. Unlike DGI, PolyGCL applies this contrastive approach to both high- and low-frequency embeddings extracted with shared-weight polynomial filters.

Subg-Con \cite{Jiao2020SubgraphCF} extends DGI by performing contrastive learning at the subgraph level. 
It selects anchor nodes and extracts subgraphs using the personalized PageRank algorithm, adjusting the agreement between anchor nodes and their corresponding subgraphs for positive and negative pairs. 
However, methods like DGI and Subg-Con rely on a readout embedding to represent entire graphs, which disregards structural information. 
GSC \cite{GSC} addresses this limitation by applying subgraph-level contrast using Wasserstein and Gromov-Wasserstein distances from OT to measure subgraph similarity, ensuring a more structurally-aware contrastive learning process.  From another point of view, FOSSIL \cite{sangare2025a} uses the fused Gromow-Wasserstein distance \cite{titouan2019optimal,brogat2022learning} in the loss function to benefit from both node and subgraph-level features.

\method~leverages OT distance metrics to effectively measure subgraph dissimilarity as in \cite{GSC}.
However, unlike previous OT-based models such as GSC \cite{GSC}, our approach introduces a novel mapping of subgraphs into a structured Gaussian space.
This design choice is driven by the properties of Gaussian embeddings, which enhance representation quality.
Our work provides both theoretical justification and empirical validation for the effectiveness of this approach.



\section{Preliminaries}
\subsection{Mathematical Notation}
Consider an undirected graph \( G = (\mathcal{V}, \mathcal{E}) \) with vertex set \( \mathcal{V} \) and edge set \( \mathcal{E} \). The feature matrix \( \mathbf{X} = [\mathbf{x}_1, \dots, \mathbf{x}_N]^\top \in \mathbb{R}^{N \times C} \) contains node features \( \mathbf{x}_i \in \mathbb{R}^C \), where \( N \) is the number of nodes and \( C \) is the feature dimension.
The adjacency matrix \( \mathbf{A} \in \mathbb{R}^{N \times N} \) represents the graph topology, and \( \mathbf{D} \) is the diagonal degree matrix.
For the \( i \)-th node, let \( G^i = (\mathcal{V}^i, \mathcal{E}^i) \) be its induced Breadth-First Search \cite{Bundy1984} (BFS) subgraph with \( k^i \) nodes with adjacency matrix \( \mathbf{A}^i \in \mathbb{R}^{k^i \times k^i} \) and feature matrix \( \mathbf{X}^i \in \mathbb{R}^{k^i \times C} \).
Our method embeds this subgraph (with the same sets of nodes and edges) producing adjacency matrix \( \tilde{\mathbf{A}}^i \) and feature matrix \( \tilde{\mathbf{X}}^i \in\mathbb{R}^{k^i \times F}\).
In this work, we preserve the subgraph topology, so that \( \tilde{\mathbf{A}}^i =  \mathbf{A}^i\).

The KL divergence \cite{van2014renyi} is an asymmetry measure between two probability distributions \( P \) and \( Q \). It quantifies the informational loss that occurs when distribution \(Q \) is utilized to approximate distribution \(P \). The KL divergence is defined as:
\begin{equation}
    \label{eqn:KL_divergence}
    D_{KL}(P \| Q) = \sum_{x \in \mathcal{X}} P(x) \log \left( \frac{P(x)}{Q(x)} \right),
\end{equation}
where \( P(x) \) and \( Q(x) \) are the probability masses of \( P \) and \( Q \) at each point \( x \) in the sample space \( \mathcal{X} \).

\subsection{Problem Formulation for Self-Supervised Graph Representation}
The goal of self-supervised graph representation learning is to learn graph embeddings \( \mathbf{R} \) through an encoder \( \varepsilon: \mathbb{R}^{N \times N} \times \mathbb{R}^{N \times C} \to \mathbb{R}^{N \times F} \), where \( \mathbf{R} = \varepsilon(\mathbf{A}, \mathbf{X}; \boldsymbol{\theta}) \) is parametrized by some learnable parameters $\boldsymbol{\theta}$ and \( F \) represents the dimension of the embeddings (representation).
This procedure is unsupervised, \ie it does not use labels.
In this paper, $\varepsilon(\cdot)$ is a GNN \cite{ju2024comprehensive}, aiming to effectively capture both the graph's feature and topology information within the representation space.

\subsection{Optimal Transport Distance}
The Wasserstein distance \cite{wasserstein}, commonly used in OT, serves as a robust metric to compare the probability distributions defined over metric spaces.
For subgraphs \( G^i \) and \( G^j \), their corresponding feature matrices are denoted as \( \mathbf{X}^i \in \mathbb{R}^{k^i \times C} \) and \( \mathbf{X}^j \in \mathbb{R}^{k^j \times C} \).
\( \mathbf{x}_{m}^{i} \in \mathbb{R}^{C}\) and \( \mathbf{x}_{n}^{j} \in \mathbb{R}^{C}\) respectively denote the feature vector of the  \(m\)-th and \(n\)-th node in the subgraphs \(G^{i}\) and \(G^{j}\), where  \(m=1,2,\ldots,k^{i}\) and \(n=1,2,\ldots,k^{j}\).
The $r$-Wasserstein distance between the feature distributions of these subgraphs is defined as \cite{kolouri2017optimal,villani2021topics}:
\begin{equation}
W_r(\mathbf{X}^i, \mathbf{X}^j) := \left(\min_{\mathbf{T} \in \pi(u, v)} \sum_{m=1}^{k^i} \sum_{n=1}^{k^j} \mathbf{T}_{(m,n)} d(\mathbf{x}^i_m, \mathbf{x}^j_n)^r\right)^{\frac{1}{r}},
\end{equation}
where \( \pi(u, v) \) represents the set of all valid possible transport plans with probability distributions \( u \) and \( v \) responsible for generating $\mathbf{x}^i_m$ and $\mathbf{x}^j_n$, respectively. These distributions capture the node feature distributions in subgraphs \( G^i \) and \( G^j \).
The matrix \( \mathbf{T} \in \pi(u, v) \) is the OT plan that matches the node pairs of the two subgraphs. \( \mathbf{T}_{(m,n)} \) is value of the transportation plan between nodes \( m \) and \( n \), and \( d(\mathbf{x}^i_m, \mathbf{x}^j_n) \) represents a valid distance metric.

\begin{figure}[t]
\centering
\includegraphics[width=0.95\textwidth]{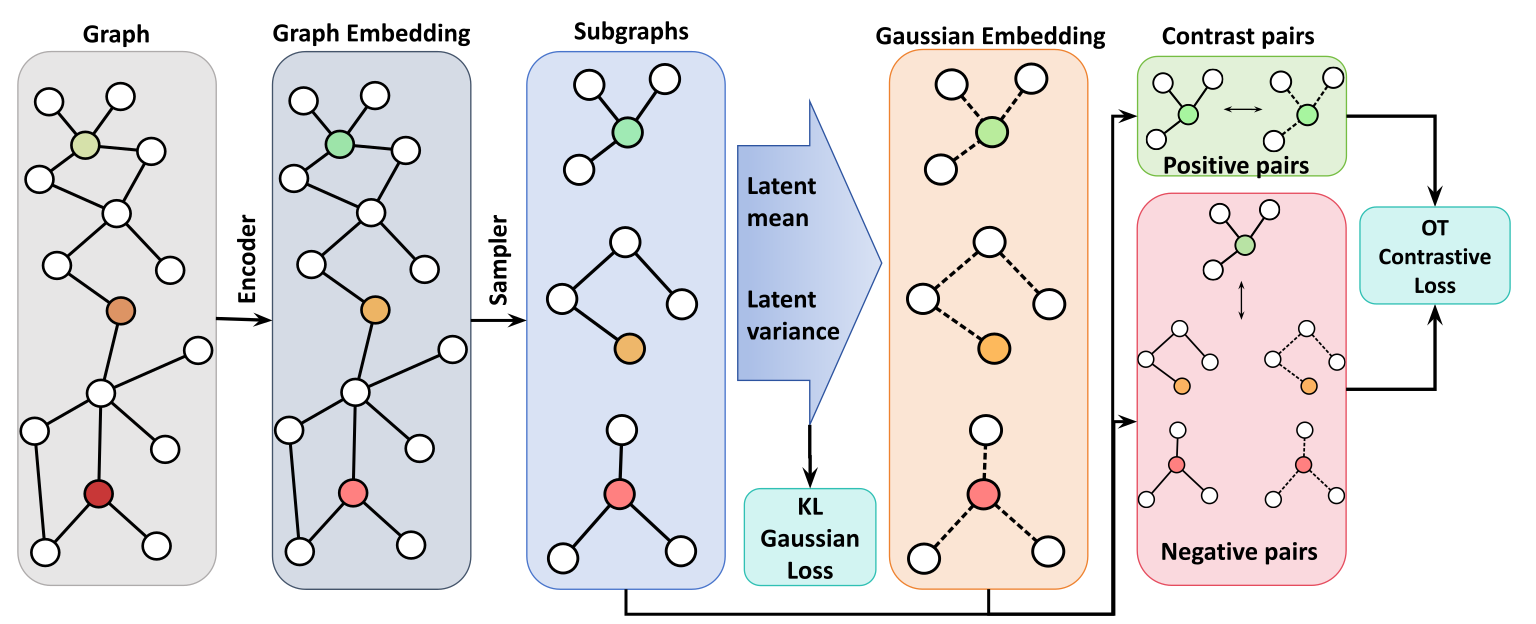}
\caption{Overview of the \method~method. Our model employs a graph encoder to obtain graph embeddings. We randomly select a set of nodes and then extract corresponding subgraphs using BFS sampling. Therefore, we use the proposed subgraph Gaussian embedding module using a KL loss to generate contrastive samples. Finally, we leverage OT distances for contrastive learning.}
\label{fig:method_overview}
\end{figure}

Similarly, the Gromov-Wasserstein distance \cite{arya2024gromovwassersteindistancespheres,vayerfused} extends this idea to compare graph-structured data, where internal distances between nodes are taken into account. For two subgraphs \( G^i \) and \( G^j \) with adjacency matrices \( \mathbf{A}^i\) and \( \mathbf{A}^j \), and feature matrices \( \mathbf{X}^i \) and \( \mathbf{X}^j \), the Gromov-Wasserstein distance is defined as \cite{vayerfused}:
\begin{equation}
\begin{split}
&GW_r(\mathbf{A}^i, \mathbf{A}^j, \mathbf{X}^i, \mathbf{X}^j)\\
&:= \left(\min_{\mathbf{T} \in \pi(u, v)} \sum_{m,\tilde{m},n,\tilde{n}} \mathbf{T}_{(m,n)} \mathbf{T}_{(\tilde{m},\tilde{n})} \left| d(\mathbf{x}^i_m, \mathbf{x}^i_{\tilde{m}})^r - d(\mathbf{x}^j_n, \mathbf{x}^j_{\tilde{n}})^r \right|\right)^{\frac{1}{r}},
\end{split}
\end{equation}
where \( d(\mathbf{x}^i_m, \mathbf{x}^i_{\tilde{m}}) \) and \( d(\mathbf{x}^j_n, \mathbf{x}^j_{\tilde{n}}) \) represent valid distance metrics between node pairs \( (m, \tilde{m}) \) in subgraph \( G^i \), and \( (n, \tilde{n}) \) in subgraph \( G^j \), respectively. Note that the node neighborhoods are considered by the term $\mathbf{T}$, thus relying on the graph topology.  

In this work, for both the Wasserstein and Gromov-Wasserstein distances, we set $r=1$ and define \( d(\mathbf{x}^i_m, \mathbf{x}^j_n)= \exp\left( -\frac{\langle \mathbf{x}^i_m, \mathbf{x}^j_n \rangle}{\tau} \right) \), where \( \langle \cdot, \cdot \rangle \) denotes the cosine similarity between node features, and \( \tau \) is a temperature parameter.

\section{Subgraph Gaussian Embedding Contrast (\method)}

Figure~\ref{fig:method_overview} shows an overview of our methodology, where our process begins with an encoder of the input graph.
Subsequently, we obtain subgraphs utilizing BFS sampling.
The embedded node representations within these subgraphs are thus embedded into a Gaussian latent space, enforced by the KL divergence regularization.
Finally, we use the Wasserstein and Gromov-Wasserstein distances to measure the dissimilarities in the subgraphs for contrastive learning.
Our methodology is described in more detail in the following sections.

\subsection{Graph Encoder}
We begin by employing a graph encoder to preprocess the graph data \cite{VAE,GSC}. The output feature matrix of the graph encoder is the desired graph representation. The graph encoder comprises some graph convolution layers. 
Further details on the implementation of these layers are provided in the Appendix \ref{Encoder}.

\subsection{Subgraph Gaussian Embedding (SGE)}

Constructing positive and negative pairs is crucial in graph contrastive learning \cite{ju2024towards}. 
The SGE module offers diversity to prevent mode collapse \cite{jing2022understandingdimensionalcollapsecontrastive}.
It also avoids generated subgraphs from becoming overly similar to the input subgraphs \cite{grill2020bootstraplatentnewapproach}. 
The SGE module comprises a GraphSAGE \cite{liu2020graphsage,hamilton2017inductive} network and then two GAT \cite{velivckovic2017graph} models, representing the mean and variance for the KL loss.
The first step in SGE is as follows:
\begin{equation}
    \mathbf{H}_\text{GSA} = \text{GraphSAGE}\left( \mathbf{H}_{\text{conv}}, \mathbf{A}\right),
\end{equation}
where \(\mathbf{H}_{\text{conv}}\) represents the output of the graph encoder. 
Following GraphSAGE, GAT employs its attention mechanism to assign weights to the relationships between each node and its neighbors.
The hidden means and variances are managed by separate GAT networks and processed as follows:
\begin{equation}
\boldsymbol{\mu} = \text{GAT}_{\boldsymbol{\mu}}\left( \mathbf{H}_\text{GSA}, \mathbf{A}\right), \quad \log \boldsymbol{\sigma} = \text{GAT}_{\boldsymbol{\sigma}}\left( \mathbf{H}_\text{GSA}, \mathbf{A}\right).
\end{equation}
In this configuration, $\boldsymbol{\mu}$ and $\log \boldsymbol{\sigma}$ are matrices of mean and variance vectors $\boldsymbol{\mu}_i$ and $\boldsymbol{\sigma}_i$ for $i=1,\hdots,N$, respectively.
In our approach, the embedded matrix \( \tilde{\mathbf{X}} \) is generated using the reparametrization trick \cite{vgae} to facilitate the differentiation and optimization of our model as follows:
\begin{equation}
\tilde{\mathbf{X}} = \boldsymbol{\mu} + \boldsymbol{\sigma} \odot \boldsymbol{\epsilon},
\end{equation}
where $\odot$ states the element-wise multiplication, and the matrix $\boldsymbol{\epsilon} = [\boldsymbol{\epsilon}_1,\hdots,\boldsymbol{\epsilon}_N]^\top$, where $\boldsymbol{\epsilon}_i\sim \mathcal{N}(\mathbf{0}, \mathbf{I})$ for $i=1,\hdots,N$, represents Gaussian (normal) noise. 

\subsection{Kullback-Leibler Gaussian Regularization Loss}
In our approach, we introduce a regularization to the SGE module to guide the embedded subgraph node features toward a Gaussian distribution.
This regularization is implemented using the KL divergence. 
The prior \( p(\tilde{\mathbf{X}})=\prod_{i=1}^{N}{p(\tilde{\mathbf{x}}_i)} \) is taken as a product of independent normal distributions for each latent variable \( \tilde{\mathbf{x}}_i \), \ie the embedded feature vector of the $i$-th node. Similarly, by benefiting from Gaussianity on the posterior distribution $q(\tilde{\mathbf{x}}_i | \mathbf{X}, \mathbf{A})=\mathcal{N}(\boldsymbol{\mu}_i, \diag(\boldsymbol{\sigma}^2_i))$ \cite{kipf2016semi}, we express it on the whole data as:
\begin{equation}
q(\tilde{\mathbf{X}} | \mathbf{X}, \mathbf{A}) = \prod_{i=1}^{N}{q(\tilde{\mathbf{x}}_i | \mathbf{X}, \mathbf{A})}=\prod_{i=1}^{N}{\mathcal{N}(\boldsymbol{\mu}_i, \diag(\boldsymbol{\sigma}^2_i))},
\end{equation}
where $\diag(\mathbf{a})$ is a diagonal matrix with the elements of the vector $\mathbf{a}$ on its main diagonal, and $\boldsymbol{\sigma}^2_i$ obtains by element-wise power operation on the vector . The expression for the regularization then simplifies to (details in the Appendix \ref{KL}):
\begin{equation}
\label{eq:overall_loss}
\mathcal{L}_R=\beta~  \text{KL}\left(q(\tilde{\mathbf{X}} | \mathbf{X}, \mathbf{A}) \| p(\tilde{\mathbf{X}})\right).
\end{equation}
Here, \( \beta \) is a hyperparameter modulating the influence of the regularization term relative to the contrastive loss, which we introduce in Section \ref{contrastive}, enabling precise control over the balance between data fidelity and distribution alignment.

\subsection{Optimal Transport Contrastive and Model Loss}
\label{contrastive}

In terms of the architectures available for the contrastive learning loss function, options include the Siamese network loss \cite{he2018twofold}, the triplet loss \cite{hermans2017defense}, and the noise contrastive estimation loss \cite{gutmann2010noise}. 
Given the presence of multiple sets of negative pairs in our model, we opt for the InfoNCE loss \cite{oord2019representationlearningcontrastivepredictive}. Our contrastive loss function integrates the Wasserstein and Gromov-Wasserstein distances into the InfoNCE loss formulation \cite{oord2019representationlearningcontrastivepredictive}, addressing the complexities of graph-based data. 
The Wasserstein distance captures feature distribution representation within subgraphs. Furthermore, the Gromov-Wasserstein distance captures structural discrepancies, providing a topology-aware similarity measure. We define \(W_{(\tau)}(\mathbf{X}^i, \tilde{\mathbf{X}}^i):=W(\mathbf{X}^i, \tilde{\mathbf{X}}^i)/\tau\), and \(GW_{(\tau)}(\mathbf{A}^i, \mathbf{X}^i, \mathbf{A}^i, \tilde{\mathbf{X}}^i):=GW(\mathbf{A}^i, \mathbf{X}^i, \mathbf{A}^i, \tilde{\mathbf{X}}^i)/\tau\), where $\tau$ is a temperature hyperparameter.
The Wasserstein (\(\mathcal{L}_{\text{W}}\)) and Gromov-Wasserstein (\(\mathcal{L}_{\text{GW}}\)) contrastive losses are given as follows:
\begin{equation}
\label{eq:InfoNCE_loss}
\begin{split}
&\mathcal{L}_{\text{W}} = - \sum_{i \in \mathcal{S}} \log \frac{e^{-W_{(\tau)}(\mathbf{X}^i, \tilde{\mathbf{X}}^i)}}{\sum_{j \in \mathcal{S}, j\neq i}^{N} \left(e^{-W_{(\tau)}(\mathbf{X}^i, \tilde{\mathbf{X}}^j)}+ e^{-W_{(\tau)}(\mathbf{X}^i, \mathbf{X}^j)}\right)},\\
&\mathcal{L}_{\text{GW}}= - \sum_{i \in \mathcal{S}} \log \frac{e^{-GW_{(\tau)}(\mathbf{A}^i,\mathbf{X}^i, \mathbf{A}^i, \tilde{\mathbf{X}}^i)}}{\sum_{j\in \mathcal{S}, j\neq i}^{N} \left(e^{-GW_{(\tau)}(\mathbf{A}^i,\mathbf{X}^i, \mathbf{A}^j, \tilde{\mathbf{X}}^j)}+ e^{-GW_{(\tau)}(\mathbf{A}^i,\mathbf{X}^i,\mathbf{A}^j,\mathbf{X}^j)}\right)},
\end{split}
\end{equation}
where $\mathcal{S}$ is the set of sampled nodes.
The model loss \( \mathcal{L} \) incorporates the contrastive and regularization components as follows:
\begin{equation}
\label{eq:overall_loss}
\mathcal{L}=  \alpha \mathcal{L}_{\text{W}} + (1 - \alpha) \mathcal{L}_{\text{GW}} + \mathcal{L}_{\text{R}},
\end{equation}
where \( \alpha \) is a hyperparameter that balances the emphasis on feature distribution and structural fidelity.

\subsection{Theoretical Analysis of the Loss Function}
\label{sec:theoretical_analysis}

The following theorem illustrates the effect of adding the term $\text{KL}(\cdot)$ to the overall loss function $\mathcal{L}$ in (\ref{eq:overall_loss}) with the input $x$ and latent variable $z$.
\begin{theorem}
\label{thm:Mu_KL}
By increasing the number of subgraphs (and consequently their associate node feature matrices), minimizing InfoNCE loss $\mathcal{L}_{W}(\cdot)$ in (\ref{eq:InfoNCE_loss}) and also the KL divergence in (\ref{eq:overall_loss}), the \method~model implicitly minimizes:
\begin{equation}
    \mathbb{E}_{\mathbf{X}\sim p(\mathbf{X}|\tilde{\mathbf{X}})}\left[\text{KL}\left(q_{\phi}(\tilde{\mathbf{X}}|\mathbf{X},\mathbf{A})||p(\tilde{\mathbf{X}}|\mathbf{X},\mathbf{A})\right)\right].
\end{equation}
\end{theorem}
\begin{proof}
Firstly, the following theorem from \cite{oord2019representationlearningcontrastivepredictive} outlines the relationship between minimizing the InfoNCE loss and maximizing mutual information between the input $x$ and latent variable $z$, \ie $I(x,z)$.
\begin{proposition}[From \cite{oord2019representationlearningcontrastivepredictive}] 
The equivalence of maximizing the mutual information between the input $x$ and latent variable $z$ and minimizing $\mathcal{L}_{\text{InfoNCE}(N)}(x,z)$ becomes tighter by increasing the number of input data $N$ as:
\begin{equation}
I(x,z)\ge\log(N)-\mathcal{L}_{\text{InfoNCE}(N)}(x,z).
\end{equation}
\end{proposition}

\noindent Next, by minimizing $\text{KL}\left(q_{\phi}(z| x)||p(z)\right)$ leading to $q_{\phi}(z|x)\approx p(z)$, one can write:
\begin{equation}
\begin{split}
&{I(\textit{x},\textit{z})=\int\int p(x,z) \log\left(\frac{p(x,z)}{p(x)p(z)}\right)\,dx\,dz= \int\int \overbrace{p(x,z)}^{p(x|z)p(z)} \log\left(\frac{p(z|x)}{p(z)}\right)\,dx\,dz}\\
&\resizebox{\textwidth}{!}{$\displaystyle{=\int p(x|z)\overbrace{\left[\int q_{\phi}(z|x) \log\left(\frac{p(z|x)}{q_{\phi}(z|x)}\right)\,dz\right]}^{-\text{KL}\left(q_{\phi}(z|x)||p(z|x)\right)}\,dx=-\mathbb{E}_{x\sim p(x|z)}\left[\text{KL}\left(q_{\phi}(z|x)||p(z|x)\right)\right].}$}
\end{split}
\end{equation}
where we have used the mathematical expectation formula $\mathbb{E}_{x\sim p(x)}\left[f(x)\right]=\int f(x)p(x)dx$ for the last equality. Therefore, by increasing the number of inputs, minimizing $\mathcal{L}_{\text{InfoNCE}(N)}(x,z)$ and also KL divergence $\text{KL}\left(q_{\phi}(z|x)||p(z)\right)$, the network implicitly minimizes $\mathbb{E}_{x\sim p(x|z)}\left[\text{KL}\left(q_{\phi}(z|x)||p(z|x)\right)\right]$, which means that the average distance over the samples from $p(x|z)$ between the parametric probability distribution $q_{\phi}(z|x)$ and $p(z|x)$ is minimized. Now, by replacing $\mathcal{L}_{\text{InfoNCE}(N)}(x,z)$, $q_{\phi}(z|x)$, $p(z)$, $p(z|x)$, and $p(x|z)$ with $\mathcal{L}_{W}$, $q_{\phi}(\tilde{\mathbf{X}}|\mathbf{X},\mathbf{A})$, $p(\tilde{\mathbf{X}})$, $p(\tilde{\mathbf{X}}|\mathbf{X},\mathbf{A})$, and $p(\mathbf{X}|\tilde{\mathbf{X}})$, respectively, the proof is completed.
\end{proof}

\method~is driven by two key principles: (i) maximizing the mutual information between the input and latent variables and (ii) designing a robust encoder that generates latent embeddings closely aligned with the true latent distribution.
Theorem \ref{thm:Mu_KL} formally establishes that enforcing the joint minimization of the OT and KL losses in the overall loss (\ref{eq:overall_loss}) leads to the minimization of the expected KL divergence $\mathbb{E}_{\mathbf{X} \sim p(\mathbf{X} | \tilde{\mathbf{X}})}\left[\text{KL}\left(q_{\phi}(\tilde{\mathbf{X}} | \mathbf{X}, \mathbf{A}) \| p(\tilde{\mathbf{X}} | \mathbf{X}, \mathbf{A})\right)\right]
$, ensuring an accurate estimation of the true conditional distribution \( p(\tilde{\mathbf{X}} | \mathbf{X}, \mathbf{A}) \). Simultaneously, this optimization strategy increases the mutual information between the input \( \mathbf{X} \) and the latent embedding \( \tilde{\mathbf{X}} \), thereby reinforcing the encoder’s capacity to preserve essential input characteristics. Theorem \ref{thm:Mu_KL} thus provides the theoretical foundation for \method’s design. Moreover, it highlights a crucial insight: minimizing the KL divergence alone does not necessarily maximize mutual information and may result in suboptimal performance, an observation we empirically validate in Section \ref{sec:validation}.  

\section{Experimental Evaluation}

In this section, we present the empirical assessment of \method~by comparing its performance against current state-of-the-art methodologies across various public datasets.
Additionally, through ablation studies, we verify the efficacy of our method. 
These studies analyze the contribution of individual \method~components to the overall performance. 
Finally, we analyze the computational cost to show our method's scalability to larger graphs.
We also explore the sensitivity of the loss balance hyperparameter \( \beta \) and the size of the subgraph on the model's performance in Appendix \ref{Sensitivity Analysis}.


\textbf{Datasets.} 
We select several widely used datasets for graph node classification to evaluate \method.
These datasets encompass various types of networks, including academic citation networks, collaboration networks, and web page networks, providing diverse challenges and characteristics.
Table \ref{tab:selected_datasets} summarizes the basic statistics of these datasets.

\begin{table}[t]
\centering
\caption{Overview of selected datasets used in the study.}
\vspace{5pt}
\label{tab:selected_datasets}
\setlength{\tabcolsep}{3pt}
\begin{tabular}{lcccccc}
\toprule
\textbf{Dataset}    & \textbf{Nodes} & \textbf{Edges}   &  \textbf{Features} & \textbf{Avg. degree} & \textbf{Classes} \\ 
\midrule
Cora \cite{kipf2016semi}       & 2,708     & 5,429       & 1,433        & 4.0            & 7          \\
Citeseer \cite{citeseer}   & 3,312    & 4,732      & 3,703       & 2.9            & 6          \\
Pubmed \cite{Pubmed}     & 19,717   & 44,338    & 500         & 4.5            & 3          \\
Coauthor \cite{coauther}   & 18,333   & 163,788     & 6,805       & 17.9            & 15         \\
Squirrel \cite{Squirrel}   & 5,201    & 217,073    & 2,089       & 83.5           & 5          \\
Chameleon \cite{Chameleon}  & 2,277    & 36,101     & 2,325       & 31.7           & 5          \\
Cornell \cite{Cornell}    & 183      & 298        & 1,703       & 3.3            & 5          \\
Texas \cite{Cornell}      & 183      & 325        & 1,703       & 3.6            & 5          \\
\bottomrule
\end{tabular}
\end{table}

\textbf{Implementation details.}
We implement \method~using \texttt{PyG} and \texttt{PyTorch}.
Our approach adopts a self-supervised scheme evaluated via linear probing.
The model is trained using the official training subsets of the referenced datasets.
Hyperparameter tuning involves a random search on the validation set to determine optimal values for the hyperparameters.
The best configuration in validation is subsequently employed for tests on the dataset.
We train our model with the Adam optimizer.
We train our models on GPU architectures, including the RTX 3060 and A40.

\textbf{Hyperparameter random search.}
Informed by the findings of \cite{gasteiger2022diffusionimprovesgraphlearning}\cite{topping2022understandingoversquashingbottlenecksgraphs}\cite{giraldo2023trade}, which indicate the sensitivity of GNNs to hyperparameter settings, we undertake random searches for hyperparameter optimization.
The training proceeds on the official splits of the train datasets, with the random search conducted on the validation dataset to pinpoint the best configurations.
These settings are then implemented to evaluate the model on the test dataset.
The ranges of the hyperparameters explored and our code implementation are available\footnote{\url{https://github.com/ShifengXIE/SubGEC/tree/main}} and will be made public after acceptance to facilitate replication and further research.

\subsection{Classification Results}

In our study, we compared our model against five state-of-the-art self-supervised node classification algorithms: POLYGCL \cite{polygcl}, GREET \cite{GREET}, GRACE \cite{GRACE}, GSC \cite{GSC}, and MUSE \cite{MUSE}.
Additionally, we include three classic SSL algorithms for a comprehensive comparison: DGI \cite{DGI}, GCA \cite{GCA}, and GraphMAE \cite{graphmae}. To provide a broader context, we also report the training results from two supervised learning models: GCN \cite{kipf2016semi} and GAT \cite{velivckovic2017graph}. 
\begin{table}[t]
\centering
\caption{Performance comparison of self-supervised and supervised graph representation learning methods across eight benchmark datasets.}
\label{tab:performance}
\vspace{5pt}
\resizebox{\linewidth}{!}{
\begin{tabular}{@{}lSSSSSSSS@{}}
\toprule
\textbf{Method} & {\textbf{Cora}} & {\textbf{Citeseer}} & {\textbf{Pubmed}} & {\textbf{Coauthor}} & {\textbf{Squirrel}} & {\textbf{Chameleon}} & {\textbf{Cornell}} & {\textbf{Texas}} \\
\midrule
\makecell[l]{GCN } & {81.40$_{\pm 0.50}$} & {70.30$_{\pm 0.50}$} & {76.80$_{\pm 0.70}$} & {93.03$_{\pm 0.31}$} & {53.43$_{\pm 1.52}$} & {64.82$_{\pm 2.32}$} & {60.54$_{\pm 3.30}$} & {67.57$_{\pm 4.80}$} \\
\makecell[l]{GAT} & {83.00$_{\pm 0.52}$} & {72.50$_{\pm 0.30}$} & {79.00$_{\pm 0.24}$} & {92.31$_{\pm 0.24}$} & {42.72$_{\pm 3.27}$} & {63.90$_{\pm 2.19}$} & {76.00$_{\pm 3.63}$} & {78.87$_{\pm 3.78}$} \\
\midrule
MUSE      & {69.90$_{\pm 0.41}$} & {66.35$_{\pm 0.40}$} & {79.95$_{\pm 0.59}$} & {90.75$_{\pm 0.39}$} & {40.15$_{\pm 3.04}$} & {55.59$_{\pm 2.21}$} & {83.78$_{\pm 3.42}$} & {83.78$_{\pm 2.79}$} \\
POLYGCL  & \textbf{84.89$_{\pm 0.62}$} & \textbf{76.28$_{\pm 0.85}$} & {81.02$_{\pm 0.27}$} & {93.76$_{\pm 0.08}$} & {55.29$_{\pm 0.72}$} & \textbf{71.62$_{\pm 0.96}$} & {77.86$_{\pm 3.11}$} & {85.24$_{\pm 1.80}$} \\
GREET     & {84.40$_{\pm 0.77}$} & {74.10$_{\pm 0.44}$} & {80.29$_{\pm 0.24}$} & \textbf{94.65$_{\pm 0.18}$} & {39.76$_{\pm 0.75}$} & {60.57$_{\pm 1.03}$} & {78.36$_{\pm 3.77}$} & {78.03$_{\pm 3.94}$} \\
GRACE    & {83.30$_{\pm 0.74}$} & {72.10$_{\pm 0.60}$} & \textbf{86.70$_{\pm 0.16}$} & {92.78$_{\pm 0.04}$} & {52.10$_{\pm 0.94}$} & {52.29$_{\pm 1.49}$} & {60.66$_{\pm 11.32}$} & {75.74$_{\pm 2.95}$} \\
GSC       & {82.80$_{\pm 0.10}$} & {71.00$_{\pm 0.10}$} & {85.60$_{\pm 0.20}$} & {91.88$_{\pm 0.11}$} & {51.32$_{\pm 0.21}$} & {64.02$_{\pm 0.29}$} & {93.56$_{\pm 1.73}$} & {88.64$_{\pm 1.21}$} \\
DGI       & {81.99$_{\pm 0.95}$} & {71.76$_{\pm 0.80}$} & {77.16$_{\pm 0.24}$} & {92.15$_{\pm 0.63}$} & {38.80$_{\pm 0.76}$} & {58.00$_{\pm 0.70}$} & {70.82$_{\pm 7.21}$} & {81.48$_{\pm 2.79}$} \\
GCA       & {78.13$_{\pm 0.85}$} & {67.81$_{\pm 0.75}$} & {80.63$_{\pm 0.31}$} & {93.10$_{\pm 0.20}$} & {47.13$_{\pm 0.61}$} & {56.54$_{\pm 1.07}$} & {53.11$_{\pm 9.34}$} & {81.02$_{\pm 2.30}$} \\
GraphMAE  & {84.20$_{\pm 0.40}$} & {73.40$_{\pm 0.40}$} & {81.10$_{\pm 0.40}$} & {80.63$_{\pm 0.15}$} & {48.26$_{\pm 1.21}$} & {71.05$_{\pm 0.36}$} & {61.93$_{\pm 4.59}$} & {67.80$_{\pm 3.37}$} \\ 
\hdashline
\method & {83.60$_{\pm 0.10}$} & {73.14$_{\pm 0.14}$} & {84.60$_{\pm 0.10}$} & {92.34$_{\pm 0.04}$} & \textbf{56.39$_{\pm 0.57}$} & {69.14$_{\pm 1.12}$} & \textbf{94.57$_{\pm 2.13}$} & \textbf{92.38$_{\pm 0.81}$} \\
\bottomrule
\end{tabular}
}
\end{table}

The results in Table \ref{tab:performance} highlight the strong performance of \method~across diverse datasets. Our model outperforms other state-of-the-art algorithms on three out of eight benchmarks while demonstrating competitive results on the remaining datasets. 
Notably, it achieves the highest accuracy on the strongly heterophilic Squirrel, Cornell, and Texas datasets, exceeding GSC, POLYGCL, and other baselines. This suggests that the proposed design is particularly robust in settings where node connectivity patterns deviate from typical homophilic assumptions. Although POLYGCL slightly surpasses \method~on certain homophilic datasets (\eg Cora and Citeseer), \method~remains comparably strong there.
Overall, these results highlight \method’s robustness in handling heterophilic structures while maintaining strong performance on homophilic graphs, demonstrating its versatility across diverse graph topologies.

\begin{table}[t]
\centering
\caption{Ablation study on KL regularization and other components. \textbf{Reg.} denotes the type of regularization applied, with possible choices including no regularization (\XSolidBrush), KL divergence (\textbf{KL}), and dropout (\textbf{D.}). \textbf{L1} indicates whether the L1 norm was used as a reconstruction loss.
\textbf{De.} represents whether a decoder was included in the model.
\textbf{Cons.} refers to whether contrastive loss was incorporated.}
\vspace{5pt}
\label{tab:ablationKL}
\resizebox{\linewidth}{!}{
\begin{tabular}{@{}cccccccccccc@{}}
\toprule
\textbf{Reg.} & \textbf{L1} & \textbf{De.} & \textbf{Cons.} & \textbf{Cora} & \textbf{Citeseer} & \textbf{Pubmed} & \textbf{Coauthor} & \textbf{Squirrel} & \textbf{Chameleon} & \textbf{Cornell} & \textbf{Texas} \\
\midrule
\XSolidBrush & \XSolidBrush & \XSolidBrush & \CheckmarkBold     & 83.00$_{\pm 0.07}$ & 71.88$_{\pm 0.07}$ & \textbf{85.46$_{\pm 0.04}$} & 91.96$_{\pm 0.09}$ & 42.99$_{\pm 0.17}$ & 64.25$_{\pm 0.21}$ & \textbf{94.58$_{\pm 0.22}$} & 75.72$_{\pm 0.40}$ \\
\textbf{KL} & \CheckmarkBold & \XSolidBrush & \XSolidBrush   & 78.78$_{\pm 1.09}$ & 68.33$_{\pm 1.00}$ & 75.56$_{\pm 1.65}$ & 88.86$_{\pm 0.25}$     & 30.52$_{\pm 0.48}$ & 48.12$_{\pm 0.63}$ & 68.43$_{\pm 0.55}$ & 73.83$_{\pm 1.20}$ \\ 
\textbf{KL} & \XSolidBrush & \CheckmarkBold & \CheckmarkBold  & 82.80$_{\pm 0.07}$ & 73.00$_{\pm 0.08}$ & 80.26$_{\pm 0.37}$ & 92.03$_{\pm 0.13}$ & 35.47$_{\pm 0.21}$ & 60.30$_{\pm 0.21}$ & 93.56$_{\pm 0.64}$ & 88.98$_{\pm 0.52}$ \\ 
\textbf{KL} & \CheckmarkBold & \XSolidBrush & \CheckmarkBold   & 81.60$_{\pm 0.99}$ & 69.60$_{\pm 0.10}$ & 67.54$_{\pm 0.35}$ & 87.03$_{\pm 0.65}$     & 30.52$_{\pm 0.48}$ & 43.26$_{\pm 0.66}$ & 53.29$_{\pm 0.13}$ & 63.45$_{\pm 0.48}$ \\ 
\textbf{D.} & \XSolidBrush & \XSolidBrush & \CheckmarkBold   & 79.00$_{\pm 0.21}$ & 70.60$_{\pm 3.52}$ & 80.84$_{\pm 0.05}$ & 91.52$_{\pm 0.37}$     & 44.24$_{\pm 0.50}$ & 58.94$_{\pm 0.72}$ & 85.76$_{\pm 0.24}$ & 87.98$_{\pm 0.15}$ \\ 
\hdashline
\textbf{KL} & \XSolidBrush & \XSolidBrush & \CheckmarkBold & \textbf{83.60$_{\pm 0.10}$} & \textbf{73.14$_{\pm 0.14}$} & {84.60$_{\pm 0.10}$} & \textbf{92.34$_{\pm 0.04}$} & \textbf{56.39$_{\pm 0.57}$} & \textbf{69.14$_{\pm 1.12}$} & 94.57$_{\pm 2.13}$ & \textbf{92.38$_{\pm 0.81}$} \\
\bottomrule
\end{tabular}
}
\end{table}
\subsection{Ablation Studies}
\label{sec:validation}

\textbf{KL divergence and contrastive loss.}
The first ablation study is concerned with analyzing some elements of \method~such as the architectural choices, the KL regularization, and the contrastive loss.
The outcomes of this ablation study are presented in Table \ref{tab:ablationKL}.


The first row in Table \ref{tab:ablationKL} analyzes the case where we drop the KL loss from \method.
We observe that in overall the performance decreases, demonstrating the importance of the KL loss as theoretically proved in Section \ref{sec:theoretical_analysis}.
On the contrary, the second row in Table \ref{tab:ablationKL} includes only the KL loss and L1 reconstruction loss without including the contrastive loss.
This effectively models a Variational Autoencoder (VAE) type method, where we observe a loss in performance.
This result also aligns with the theoretical findings in Theorem \ref{thm:Mu_KL}, where we show that solely relying on the minimization of the KL loss does not guarantee the accurate estimation of the encoder distribution and can lead to performance degradation compared to using both the KL and contrastive loss functions.


The third model in Table \ref{tab:ablationKL} incorporates a decoder into \method, \ie we use a VAE-type architecture to generate contrastive pairs.
The decoder consists of two fully connected multi-layer perceptrons.
This model achieves competitive results only on specific databases, illustrating that \method~is not merely a combination of a VAE generative model and contrastive learning training methodologies.
The fourth model includes a norm-1 reconstruction loss, calculated as the norm of the difference between input and output features, adding a constraint to enforce similarity between input and output features.
The results indicate that enforcing such similarity is not reasonable.
Finally, the fifth model in Table \ref{tab:ablationKL} replaces the KL divergence with the commonly used regularization technique, dropout.
The results show that our method outperforms dropout.



\textbf{Contrastive loss.}
The second ablation study examines the impact of the distance metric used in our contrastive loss, specifically comparing OT distances with alternative approaches. 
Table \ref{tab:ablationOT} presents the results of this study, evaluating models with Wasserstein-only, Gromov-Wasserstein-only, and L1-only metrics, as well as \method.
Our findings indicate that excluding OT distances leads to suboptimal performance, particularly on heterophilic datasets.
Additionally, we observe that the Gromov-Wasserstein distance slightly outperforms the Wasserstein distance on homophilic datasets.
Most importantly, incorporating both Wasserstein and Gromov-Wasserstein distances in the contrastive loss consistently yields the best performance across all datasets.




\begin{table}[t]
\centering
\caption{Ablation studies on the choice of the distance metric in the contrastive loss. \textbf{W} indicates the use of the Wasserstein distance. \textbf{GW} indicates the use of the Gromov-Wasserstein distance. \textbf{L1} indicates the use of a simple L1 distance.
}
\label{tab:ablationOT}
\vspace{5pt}
\resizebox{\linewidth}{!}{
\begin{tabular}{@{}ccccccccccc@{}}
\toprule
\textbf{W} &\textbf{GW} &\textbf{L1} & \textbf{Cora} & \textbf{Citeseer} & \textbf{Pubmed} & \textbf{Coauthor} & \textbf{Squirrel} & \textbf{Chameleon} & \textbf{Cornell} & \textbf{Texas} \\
\midrule
\CheckmarkBold & \XSolidBrush & \XSolidBrush   & 77.00$_{\pm 0.81}$ & 66.80$_{\pm 1.39}$ & 78.24$_{\pm 1.22}$ & 88.65$_{\pm 0.62}$     & 49.02$_{\pm 0.85}$ & 62.50$_{\pm 0.49}$ & 91.29$_{\pm 0.10}$ & 87.67$_{\pm 0.09
}$ \\
\XSolidBrush & \CheckmarkBold & \XSolidBrush   & 76.20$_{\pm 1.56}$ & 68.98$_{\pm 0.23}$ & 80.20$_{\pm 1.42}$ & 91.08$_{\pm 0.28}$     & 45.16$_{\pm 0.55}$ & 56.17$_{\pm 0.28}$ & 90.16$_{\pm 1.52}$ & 88.47$_{\pm 0.89}$ \\
\XSolidBrush & \XSolidBrush & \CheckmarkBold  & 79.84$_{\pm 0.68}$ & 69.80$_{\pm 0.64}$ & 79.20$_{\pm 2.54}$ & 82.23$_{\pm 1.51}$     & 47.10$_{\pm 0.58}$ & 58.35$_{\pm 0.92}$ & 90.33$_{\pm 1.08}$ & 84.59$_{\pm 0.80}$ \\\hdashline
\CheckmarkBold  &\CheckmarkBold  & \XSolidBrush & \textbf{83.60$_{\pm 0.10}$} & \textbf{73.14$_{\pm 0.14}$} & {\textbf{84.60}$_{\pm 0.10}$} & \textbf{92.34$_{\pm 0.04}$} & \textbf{56.39$_{\pm 0.57}$} & \textbf{69.14$_{\pm 1.12}$} & \textbf{94.57}$_{\pm 2.13}$ & \textbf{92.38$_{\pm 0.81}$} \\
\bottomrule
\end{tabular}
}
\end{table}

\begin{figure}[!t]
    \centering
    \includegraphics[width=0.8\linewidth]{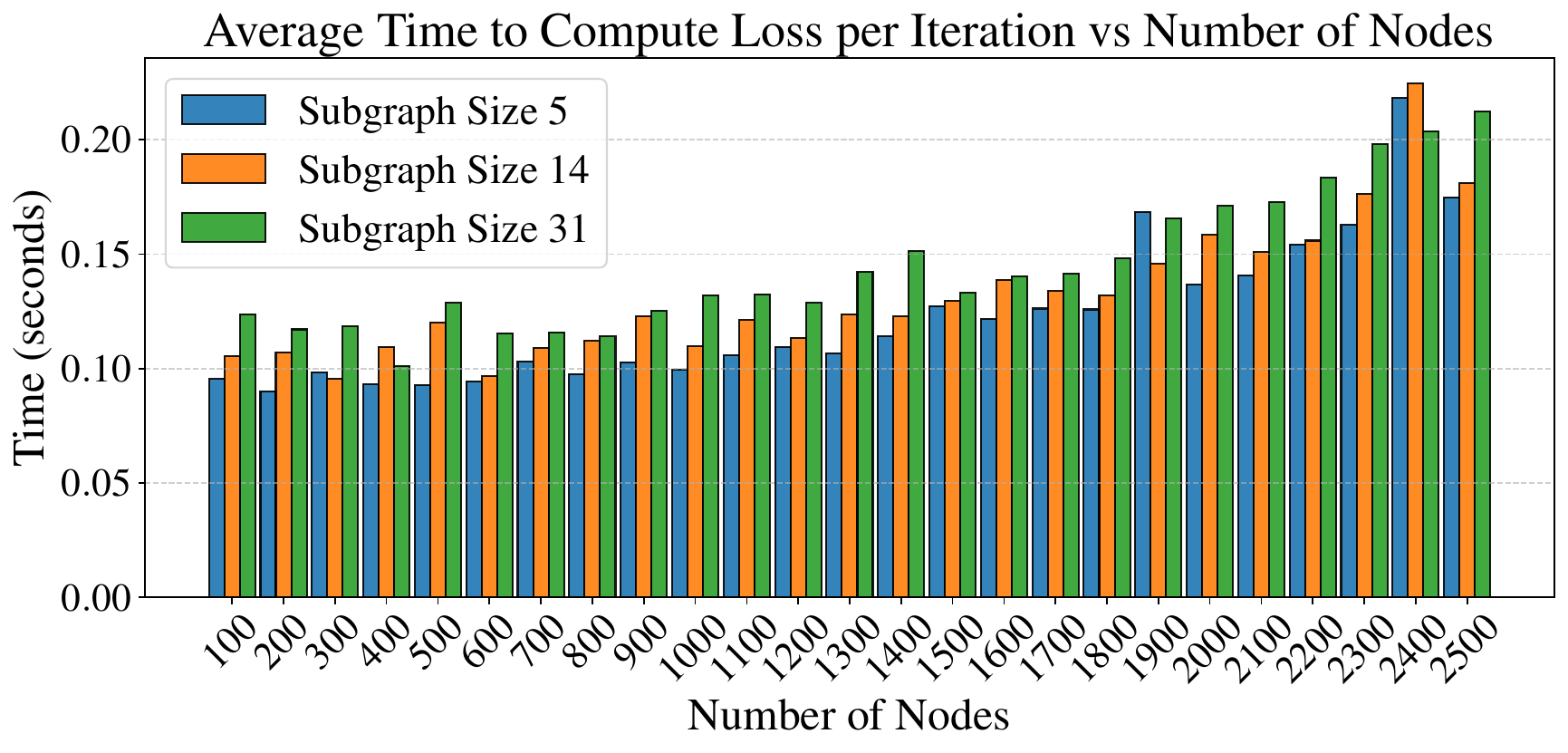}
    \caption{Average time to compute loss per iteration as a function of the number of nodes. The figure compares the computation times for three different subgraph sizes ($5$, $14$, and $31$).}
    \label{fig:Time}
\end{figure}

\subsection{Running Time}

We employ a subgraph sampling strategy to avoid the high computational complexity of OT computations.
Figure \ref{fig:Time} shows the average time to compute the loss per iteration.
The running time can vary due to server performance fluctuations, leading to non-monotonic timing variations.
We observe that the running time remains low, increasing modestly as the graph size grows from 100 to 2,500 nodes, with times ranging from 0.1 to 0.2 seconds.
We attribute this increase in computational time to the higher dimensionalities of the adjacency matrices when subgraphs are sampled.
Overall, \method~keeps a low running time even for increasing graph sizes, potentially enabling applications in large-scale graph SSL tasks.





\section{Conclusion}

This paper introduces the \method, a novel GRL framework that leverages subgraph Gaussian embeddings for self-supervised contrastive learning. 
Our approach maps subgraphs into a Gaussian space, ensuring a controlled distribution while preserving essential subgraph characteristics.
We also incorporate the OT Wasserstein and Gromov-Wasserstein distances into our contrastive loss.
From a theoretical perspective, we demonstrated that our method minimizes the KL divergence between the learned encoder distribution and the Gaussian distribution while maximizing mutual information between input and latent variables.
Our experiments on multiple benchmark datasets validate these theoretical insights and show that \method~outperforms or presents competitive performance against previous state-of-the-art models.
Our findings emphasize the importance of controlling the distribution of contrastive pairs in SSL.

\section*{Acknowledgment}
This research was supported by DATAIA Convergence Institute as part of the «Programme d’Investissement d’Avenir», (ANR-17-CONV-0003) operated by the center Hi! PARIS. This work was also supported by the ANR French National Research Agency under the JCJC projects DeSNAP (ANR-24CE23-1895-01).

\appendix

\section{Graph Convolutional Network}
\label{Encoder}
The graph encoder uses two graph convolution layers, which are mathematically represented as follows:
\begin{equation}
\mathbf{H}_{1} = \sigma\left((\mathbf{D}^{-\frac{1}{2}} \left(\mathbf{A}+ \mathbf{I}\right) \mathbf{D}^{-\frac{1}{2}}  \mathbf{X} \Theta_1\right), \quad
\mathbf{H}_2 = \sigma\left(\mathbf{D}^{-\frac{1}{2}} \left(\mathbf{A}+ \mathbf{I}\right) \mathbf{D}^{-\frac{1}{2}} \mathbf{H}_1 \Theta_2\right).
\end{equation}

\section{Details of the KL Divergence in Subgraph Gaussian Embedding}
\label{KL}

The KL divergence between these two distributions has a well-known closed-form expression. In our setting, we write \cite{VAE}:
\begin{equation}
\label{eq:kl_details}
\text{KL}\bigl(q(\tilde{\mathbf{X}} | \mathbf{X}, \mathbf{A}) \,\big\|\, p(\tilde{\mathbf{X}})\bigr) 
= \frac{1}{2\vert \mathcal{P} \vert} \sum_{i \in \mathcal{P} } \sum_{j=1}^{d} \Bigl( {\mu}_{ij}^2 + {\sigma}_{ij}^2 - 1 - 2 \,\log {\sigma}_{ij} \Bigr),
\end{equation}
where \({\mu}_{ij}\) and \({\sigma}_{ij}\) represent the \(j\)-th components of the latent mean and latent standard deviation for node \(i\). The set \(\mathcal{P}\) indexes the nodes in the induced subgraphs under consideration, and \(d\) is the dimensionality of the latent space.


\section{Sensitivity Analysis}
\label{Sensitivity Analysis}

To investigate the impact of the regularization constraint on our method, experiments were conducted on the Cora dataset. The influence of regularization within the loss function was controlled by varying the hyperparameter \(\beta\), which ranged from \(10^{-6}\) to \(10^{2}\). The results, as illustrated in Figure \ref{fig:sensitivity}, indicate sensitivity to changes in  \(\beta\). Specifically, we observe that values of \(\beta\) greater than or equal to \(10^{-5}\) have a pronounced effect on the model's performance. Optimal results on the Cora dataset are achieved when \(\beta\) was set within \(10^{-3}\).


To evaluate the sensitivity of \method~to the subgraph size hyperparameter, we conducted a sensitivity analysis on the Cora dataset using subgraph sizes \( k = 5, 15, 25, \) and \( 35 \). As shown in Figure~\ref{fig:sensitivity_subgraph_size}, the model exhibits robust performance across a wide range of subgraph sizes, with competitive mean test accuracy and low variability observed for \( k = 5 \) to \( 25 \). 
While \( k = 15 \) achieves marginally higher accuracy, the minimal differences in performance across this range suggest that the model is not overly sensitive to precise subgraph size selections. A gradual decline in performance at \( k = 35 \) highlights the upper bound of robustness, likely due to increased noise from redundant structural information. 

\begin{figure}[t]
    \centering
    \begin{subfigure}[t]{0.49\textwidth}
        \centering
        \includegraphics[width=\textwidth]{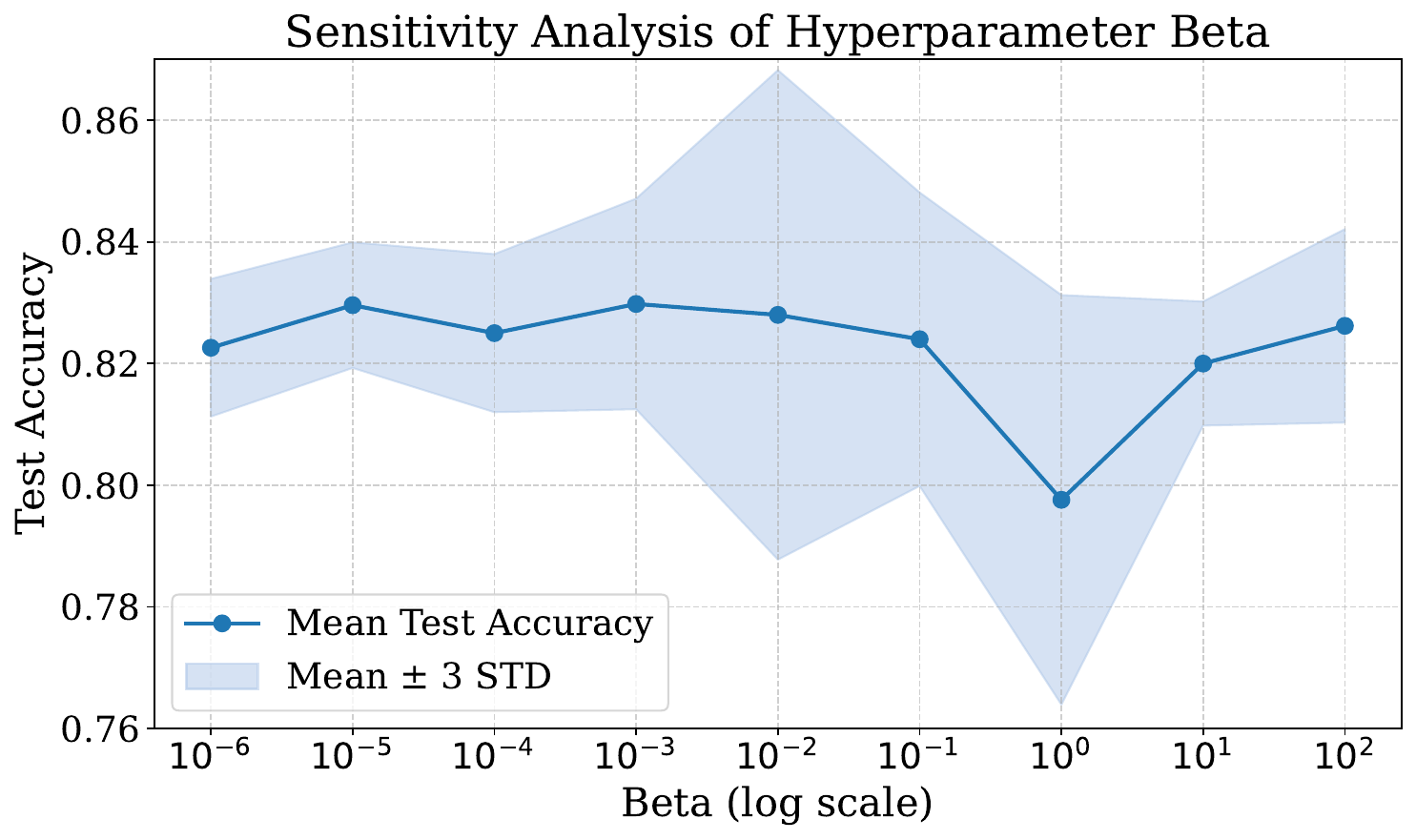}
        \caption{Sensitivity analysis of hyperparameter beta \(\beta\).}
        \label{fig:sensitivity}
    \end{subfigure}
    \hfill
    \begin{subfigure}[t]{0.49\textwidth}
        \centering
        \includegraphics[width=\textwidth]{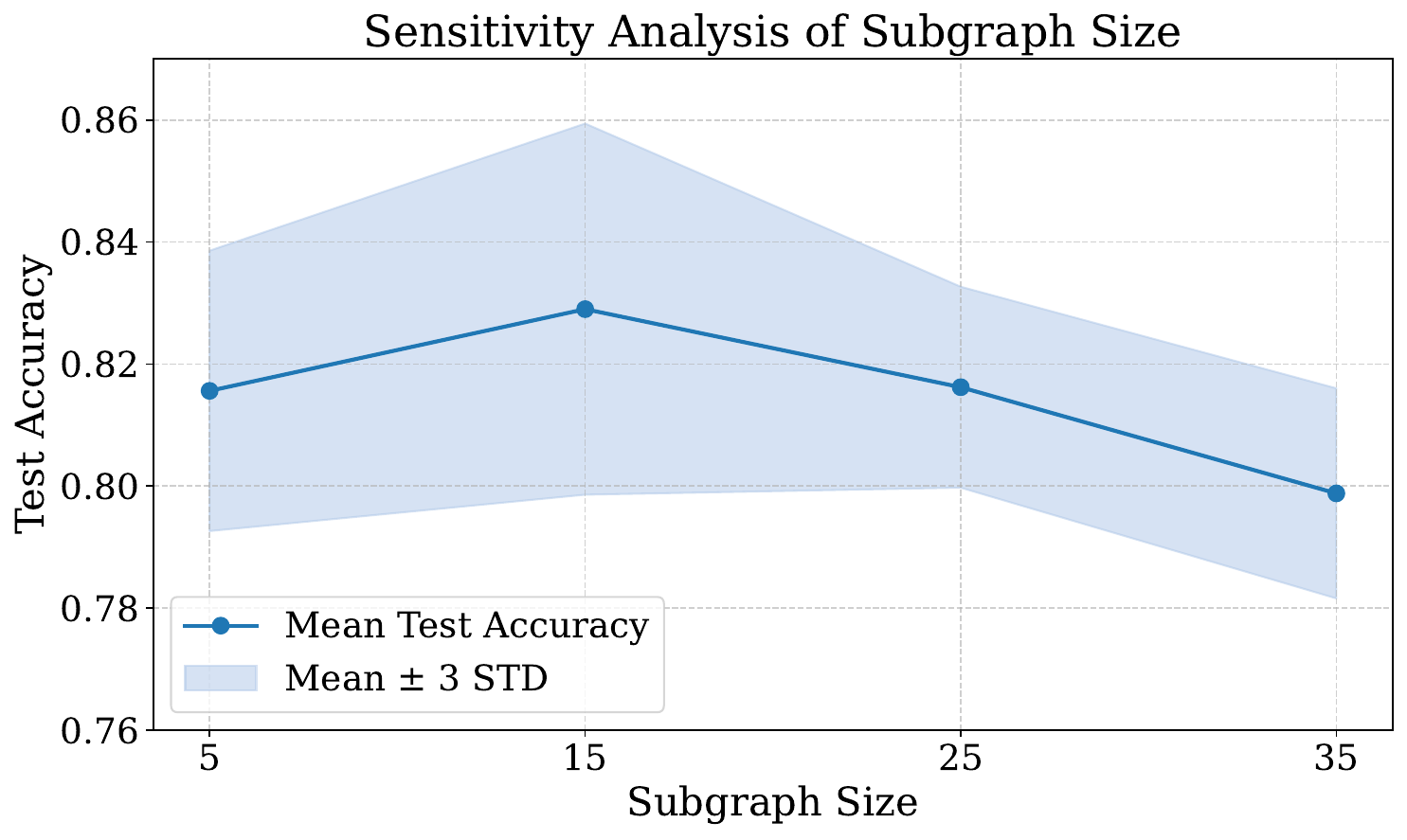}
        \caption{Sensitivity analysis of subgraph size \( k^i \).}
        \label{fig:sensitivity_subgraph_size}
    \end{subfigure}
    \caption{The plot displays the mean test accuracy (solid blue line) along with a shaded confidence region representing the mean \(\pm3\) standard deviations. The analysis illustrates the sensitivity of test accuracy to variations in hyperparameter beta and subgraph sizes.}
\end{figure}

\end{document}